\documentclass[journal]{IEEEtran}

\usepackage{siunitx,booktabs,threeparttable,caption}
\usepackage{amsmath}
\usepackage{amssymb}
\usepackage{amsthm}

\usepackage{url}

\usepackage[ruled,vlined]{algorithm2e}
\usepackage[]{algorithm2e}

\usepackage{subfig}

\usepackage{nomencl}
\makenomenclature

\usepackage{multicol}
\usepackage{multirow}

\theoremstyle{definition}
\newtheorem{definition}{Definition}
\theoremstyle{plain}
\newtheorem{theorem}{Theorem}
\newtheorem{lemma}{Lemma}

\usepackage{graphicx}
\usepackage{amssymb}

\begin{document}

\title{A Gradient Free  Neural Network Framework Based on Universal Approximation Theorem}

\author{N. P. Bakas,
        A. Langousis,
        M. Nicolaou,
        and~S. A. Chatzichristofis
\thanks{N. P. Bakas and Mihalis Nicolaou are with the Computation-based Science and Technology Research Center, The Cyprus Institute, 20 Konstantinou Kavafi Street, 2121, Aglantzia Nicosia, Cyprus. e-mail: (n.bakas@cyi.ac.cy, m.nicolaou@cyi.ac.cy}
\thanks{ A. Langousis is with the Department of Civil Engineering, University of Patras, 265 04 Patras, Greece.}
\thanks{ S. A. Chatzichristofis is with the Intelligent Systems Lab \&  Department of Computer Science, Neapolis University Pafos, 2 Danais Avenue, 8042 Pafos, Cyprus.}
\thanks{Manuscript received August 03, 2020; revised xxx}}

\markboth{Journal of \LaTeX\ Class Files,~Vol.~14, No.~8, August~2015}%
{Shell \MakeLowercase{\textit{et al.}}: Bare Demo of IEEEtran.cls for IEEE Journals}

\maketitle

\begin{abstract}
We present a numerical scheme for computation of Artificial Neural Networks (ANN) weights, which stems from the Universal Approximation Theorem, avoiding laborious iterations. The proposed algorithm adheres to the underlying theory, is highly fast, and results in remarkably low errors when applied for regression and classification of complex data-sets, such as the Griewank function of multiple variables $\mathbf{x} \in \mathbb{R}^{100}$ with random noise addition, and MNIST database for handwritten digits recognition, with $7\times10^4$ images. The same mathematical formulation is found capable of approximating highly nonlinear functions in multiple dimensions, with low errors (e.g. $10^{-10}$) for the test-set of the unknown functions, their higher-order partial derivatives,  as well as numerically solving Partial Differential Equations. The method is based on the calculation of the weights of each neuron in small neighborhoods of the data, such that the corresponding local approximation matrix is invertible. Accordingly, optimization of hyperparameters is not necessary, as the number of neurons stems directly from the dimensionality of the data, further improving the algorithmic speed. Under this setting, overfitting is inherently avoided, and the results are interpretable and reproducible. The complexity of the proposed algorithm is of class P with $\mathcal{O}(mn^2)+\mathcal{O}(\frac{m^3}{n^2})-\mathcal{O}(\log(n+1))$ computing time, with respect to the observations $m$ and features $n$, in contrast with the NP-Complete class of standard algorithms for ANN training. The performance of the method is high, irrespective of the size of the dataset, and the test-set errors are similar or smaller than the training errors, indicating the generalization efficiency of the algorithm. A supplementary computer code in Julia Language is provided, which can be used to reproduce the validation examples, and/or apply the algorithm to other datasets.
\end{abstract}


\section{Introduction}

Although Artificial Intelligence (AI) has been broadening its numerical methods and extending its fields of application, empirical rigor is not following such advancements \cite{sculley2018winner's}, with researchers questioning the accuracy of iterative algorithms \cite{Hutson2018}, as the obtained results for a certain problem are not always reproducible \cite{Hutson725,belthangady2019applications}. In theory, Artificial Neural Networks (ANN) are capable of approximating any continuous function \cite{Hassoun2005} but, apart from existence, the theory alone cannot conclude on a universal approach to calculate an optimal set of ANN model parameters, also referred to as weights. Along these lines, iterative optimization algorithms \cite{ruder2016overview} are usually applied to reach an optimal set of ANN weights ${{w}}$, which minimize the total error of model estimates. Note, however, that apart from trivial cases rarely met in practice, the optimization problem has more than one local minima, and its solution requires multiple iterations that increase significantly the computational load. To resolve this issue, enhanced optimization methods such as stochastic gradient descent \cite{bottou2010large,johnson2013accelerating}, have been proposed. Another common issue in ANN applications is that of overfitting, which relates to the selection of a weighting scheme that approximates well a given set of data while failing to generalize the accuracy of the predictions beyond the training set. To remedy overfitting problems, several methods have been proposed and effectively applied, such as dropout \cite{srivastava2014dropout}. Additional, and probably more important concerns regarding effective application of ANN algorithms, are a) the arbitrary selection of the number $N$ of computational Neurons, which may result in an unnecessary increase of the computational time, and b) the optimization of the hyper-parameters of the selected ANN architecture \cite{bergstra2011algorithms,bergstra2012random,Feurer2019}, which corresponds to solving an optimization problem with objective function values determined by the solution of another optimization problem, that is the calculation of ANN weights for a given training set.\par
The purpose of this work is to develop a numerical scheme for the calculation of the optimal weights ${{w}}$, the number of Neurons $N$, and other parameters of ANN algorithms, which relies on theoretical arguments, in our case the Universal Approximation Theorem and, at the same time, being fast and precise. This has been attained without deviating from the classical ANN representation, by utilizing a novel numerical scheme, dividing the studied data-set into small neighborhoods, and performing matrix manipulations for the calculation of the sought weights. The numerical experiments exhibit high accuracy, attaining remarkably low errors in the test-set of known datasets such as MNIST for computer vision, and complex nonlinear functions for regression, while the computational time is kept low. Interestingly, the same Algorithmic scheme may be applied to approximate the solution of Partial Differential Equations (PDEs), appearing in Physics, Engineering, Financial Sciences, etc. The paper is organized as follows. In section \ref{sec:The ANNbN Method}, we present the general formulation of the suggested method, hereafter referred to as ANNbN (Artificial Neural Networks by Neighborhoods). More precisely, the basic formulation of the ANNbN approach is progressively developed in sections \ref{sec:k-cluster}, and \ref{sec:all-obs}, Section \ref{sec:radial} extends the method for the case when radial basis functions are utilized, while Sections \ref{sec:derivatives} and \ref{sec:pdes} implement the method for approximation of derivatives, and solution of PDEs, respectively. Section \ref{sec:deep} transforms the original scheme to Deep Networks, and Section \ref{sec:ensembles} to Ensembles of ANNs. The results of the conducted numerical experiments are presented and discussed in Section \ref{sec:Validation Results}. Conclusions and future research directions are presented in Section \ref{sec:Discussion}. An open-source computer code written in Julia \cite{bezanson2017julia} programming Language is available at \url{https://github.com/nbakas/ANNbN.jl}.

\section{Artificial Neural Networks by Neighborhoods (ANNbN)}
\label{sec:The ANNbN Method}

Let ${{x}_{ij}}$ be some given data of $j\in \left\{ 1,2,\ldots ,n \right\}$ input variables in $i\in \left\{ 1,2,\ldots ,m \right\}$ observations of $y_i$ responses. The Universal Approximation Theorem \cite{Cybenko1989, Tadeusiewicz1995}, ensures the existence of an integer $N$, such that
    \[y_i \cong {{f}_{i}}({{x}_{i1}},{{x}_{i2}},\ldots ,{{x}_{in}})= \sum\limits_{k=1}^{N}{{{v}_{k}}}\sigma \left( \sum\limits_{j=1}^{n}{{{w}_{jk}}{{x}_{ij}}}+{{b}_{k}} \right)+{{b}_{0}},\]
with approximation errors $\epsilon_i=y_i-f_i$ among the given response $y_i$ and the corresponding simulated $f_i$, arbitrarily low. $N$ is the number of Neurons, ${{w}_{jk}}$ and ${{b}_{k}}$ denote the local approximation weights and bias terms, respectively, of the linear summation conducted for each neuron $k$, and ${{v}_{k}},{{b}_{0}}$ correspond to the global approximation weights and bias terms, respectively, of the linear summation upon all neurons. $\sigma$ is any sigmoid function, as presented below in Section \ref{sec:k-cluster}.
\par 
The suggested ANNbN (Artificial Neural Networks by Neighborhoods) method is based on segmentation of a given dataset into smaller clusters of data, so that each cluster $k$ is representative of the local neighborhood of $y_{ik}$ responses and, subsequently, uses the weights ${{w}_{jk}}$ calculated for each cluster to derive the global weights $v$ of the overall approximation. To conclude on the neighborhoods (i.e. the proximity clusters) of the response observations $y_{i}$, we use the well known $k$-means clustering algorithm (see e.g. \cite{macqueen1967some,hartigan1979algorithm} and $k$-means++ for the initial seed (\cite{arthur2007k}). Any other clustering algorithm can be utilized as well, while by supplying the initial seed, the obtained results are always reproducible. It is worth mentioning that the method works well even without clustering of the data. Clustering adds significant computational load, especially for large datasets, however as presented in Table \ref{tab:MNIST} for the MNIST dataset, ANNbN yields prevalent results even without clustering.

\subsection{Basic Formulation for Shallow Networks}
\label{sec:shallow}
Figure \ref{fig:ANNbN} illustrates the calculation process for the ANNbN weights. Contrary to the regular ANN approach where all responses $y_{i}$ are treated in a single step as a whole, the ANNbN method first splits the responses into proximity clusters, calculates the weights $w_{jk}$ in each cluster $k$ using the responses $y_{ik}$ and corresponding input data $x_{ijk}$, and subsequently uses the derived weights ${{w}_{jk}}$ for each cluster to calculate the global weights $v$ of the overall approximation. The aforementioned two step approach is detailed in Sub-Sections \ref{sec:k-cluster}, \ref{sec:all-obs} below.

\begin{figure*}[ht]
    \centering
    \includegraphics[scale=1.2]{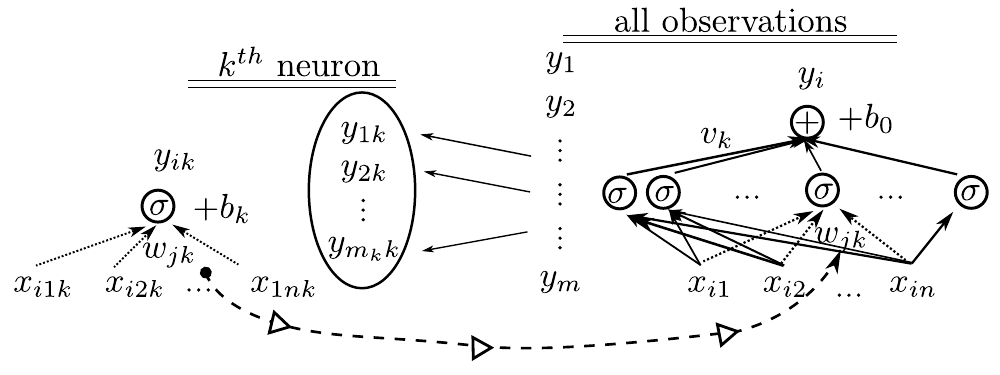}
    \caption{Illustration of the numerical procedure to calculate ANNbN local and global weights: Initial calculation of local weights $w_{jk}$ for each neuron $k$  (left panel), and subsequent calculation of the global weights $v_k$ of the entire network (right panel).}
    \label{fig:ANNbN}
\end{figure*}

\subsubsection{Calculation of $w_{jk}$ and $b_k$ in the $k^{th}$ cluster}
\label{sec:k-cluster}

Let $m_k$ be the observations found in the $k^{th}$ cluster, with $\sum_{k=1}^N m_k=m$, $\sigma$ the sigmoid function, which may be selected among the variety of sigmoids, such as $\sigma (x)={\frac {1}{1+e^{-x}}}$, and $\sigma^{-1}$ the inverted sigmoid, $ \sigma^{-1}(y)=\log \left({\frac {y}{1-y}}\right)$. Within the $k^{th}$ cluster, we may write
\newenvironment{spmatrix}[1]
 {\def\mysubscript{#1}\mathop\bgroup\begin{pmatrix}}
 {\end{pmatrix}\egroup_{\textstyle\mathstrut\mysubscript}}
\[
\begin{split}
\begin{pmatrix}
    \sigma(x_{11k}w_{1k} + x_{12k}w_{2k} + \dots + x_{1nk}w_{nk} + b_k)\\
    \sigma(x_{21k}w_{1k} + x_{22k}w_{2k} + \dots + x_{2nk}w_{nk} + b_k)\\
    \vdots \\
    \sigma(x_{{m_k}1k}w_{1k} + x_{{m_k}2k}w_{2k} + \dots + x_{{m_k}nk}w_{nk} + b_k)\\
\end{pmatrix}
= \\
\begin{pmatrix}
    y_{{1}k},   
    y_{{2}k},   
    \hdots, 
    y_{{m_k}k}
\end{pmatrix}^{T},
\end{split}
\]
and by utilizing the inverse sigmoid function $\sigma^{-1}$, and writing the left part of the Equation in matrix form, we deduce that
\begin{equation}
\begin{split}
\begin{spmatrix}{\mathbf X_k}
    x_{{1}1k} & x_{{1}2k} & \dots & x_{{1}nk} & 1\\
    x_{{2}1k} & x_{{2}2k} & \dots & x_{{2}nk} & 1\\
    \vdots &\vdots &\ddots &\vdots & 1\\
    x_{{m_k}1k} & x_{{m_k}2k} & \dots & x_{{m_k}nk} & 1\\
\end{spmatrix}
\begin{spmatrix}{\mathbf w_k}
    w_{1k} \\
    w_{2k}  \\
    \vdots\\
    w_{nk}\\
    b_k
\end{spmatrix}  \\
=
\begin{spmatrix}{\mathbf {\hat{y}}_k}
    \sigma^{-1}(y_{{1}k})  \\
    \sigma^{-1}(y_{{2}k})  \\
    \vdots\\
    \sigma^{-1}(y_{{m_k}k})
\end{spmatrix},
\end{split}
\label{eq:system-m-cluster}
\end{equation}

where $\mathbf {\hat{y}}_k =\sigma^{-1}(y_{{i}k})$, with ${i} \in \{1,2,...,m_k\}$, and $m_k$ the observations found in the $k^{th}$ cluster. For distinct observations, with $x_{{m_i}j} \neq 1 \forall m_i$, the matrix $\mathbf X_k$ is of full row rank, and if we construct clusters with $m_k = n+1$, the matrix $\mathbf X_k$ is invertible, and the system in Equation \ref{eq:system-m-cluster} has a unique solution. Hence, because the dimensions of $\mathbf X_k$ are small $(m_k<<m)$, we may rapidly calculate the approximation weights $\mathbf w_k$ (Figure \ref{fig:ANNbN} left) in the $k^{th}$ cluster (corresponding to the $k^{th}$ neuron) by
\begin{equation}
    \mathbf w_k=\mathbf X_k^{-1} \mathbf {\hat{y}}_k.
\end{equation}
If the clusters are not equally sized, we may solve numerically Equation \ref{eq:system-m-cluster}, by utilizing any appropriate algorithm for the solution of Linear systems, e.g. $\mathbf{w_k} =\mathbf{X}^{+} \mathbf{y} +(\mathbf I-\mathbf{X}^{+} \mathbf{X}) { \omega}$, where $X^{+}$ is the pseudo-inverse, and ${\omega}$ the vector of free parameters.

\subsubsection{Calculation of $v_k$ and $b_0$ exploiting all the given observations}
\label{sec:all-obs}

Following the computation of the weights $\mathbf w_k$, for each neuron $k$ in the hidden layer, we may write for all the neurons connected with the external layer that

\begin{equation}
\sigma \odot 
    \begin{spmatrix}{\mathbf O}
        \mathbf X \mathbf w_1 \quad
        \mathbf X \mathbf w_2 \quad
        \dots \quad
        \mathbf X \mathbf w_N \quad
        \mathbf 1
    \end{spmatrix}
    \begin{spmatrix}{\mathbf v}
        v_1 \\
        v_2 \\
        \vdots\\
        v_N \\
        b_0
    \end{spmatrix}
    =
    \begin{spmatrix}{\mathbf y}
        y_1  \\
        y_2  \\
        \vdots\\
        y_m
    \end{spmatrix},
    \label{eq:out-layer}
\end{equation}

where $\mathbf{1}=\left\{ 1,1,\ldots ,1 \right\}^{T}$, with length $m$, and $\mathbf X$ is the matrix containing the entire sample; in contrast with the previous step that utilized $\mathbf X_k$ containing the observations in cluster $k$. The symbol $\odot$ implies the element-wise application of $\sigma$ to each $\mathbf X \mathbf w_k$. By solving the system of Equations (\ref{eq:out-layer}), we may compute the weights $\mathbf v$. In the numerical experiments, the local approximation weights $\mathbf{w}j$ are distinct, while the number of neurons is usually smaller than the number of observations ($N<m$), hence it was found numerically fast to solve Equation \ref{eq:out-layer} by

\begin{equation}
    \mathbf v=(\mathbf O^{T}\mathbf O)^{-1} \mathbf O^{T}\mathbf y,
    \label{eq:calc-v}
\end{equation}

and obtain the entire representation of the ANNbN.

\subsection{ANNbN with Radial Basis Functions as Kernels}
\label{sec:radial}
The method was further expanded by using Radial Basis Functions (RBFs) for the approximation, $\varphi (r)$, depending on the distances among the observations $r$, instead of their raw values (Figure \ref{fig:radial}), again in the clusters of data, instead of the entire sample. A variety of studies exist on the approximation efficiency of RBFs \cite{Yiotis2015,Babouskos2015}, however they refer to noiseless data, and the entire sample, instead of neighborhoods. We should also distinguish this approach of RBFs implemented as ANNbN, with the Radial Basis Function Newtorks \cite{schwenker2001three,park1991universal}, with $\varphi ({\mathbf  {x}})=\sum _{{i=1}}^{N}a_{i}\varphi (||{\mathbf  {x}}-{\mathbf  {c}}_{i}||)$, where the centers ${\mathbf  {c}}_{i}$ are the clusters' means - instead of collocation points, $N$ is the number of neurons, and $\alpha_i$ are calculated by training, instead of matrix manipulation. In the proposed formulation, the representation regards the distances $r_{ijk}$ (Figure \ref{fig:radial}) among all the observation ${\mathbf {x}}_{ik}=\{x_{i1k},x_{i2k},\dots,x_{ink}\}$ in cluster $k$ with dimension (features) $n$, and ${i} \in \{1,2,...,m_k\}$, and another observation in the same cluster ${\mathbf {x}}_{jk}=\{x_{j1k},x_{j2k},\dots,x_{jnk}\}$, with ${j} \in \{1,2,...,m_k\}$. Accordingly, we may approximate the responses in the $k^{th}$ cluster $y_{ik}$, by

\begin{equation}
\resizebox{0.95\hsize}{!}{$
\begin{split}
    \begin{spmatrix}{\pmb\varphi_k}
    \varphi (\|\mathbf x_{1k}-\mathbf x_{1k}\|)&\varphi (\|\mathbf x_{2k}-\mathbf x_{1k}\|)&\dots &\varphi (\|\mathbf x_{{m_k}k}-\mathbf x_{1k}\|)\\
    \varphi (\|\mathbf x_{1k}-\mathbf x_{2k}\|)&\varphi (\|\mathbf x_{2k}-\mathbf x_{2k}\|)&\dots &\varphi (\|\mathbf x_{{m_k}k}-\mathbf x_{2k}\|)\\
    \vdots &\vdots &\ddots &\vdots \\
    \varphi (\|\mathbf x_{1k}-\mathbf x_{{m_k}k}\|)&\varphi (\|\mathbf x_{2k}-\mathbf x_{{m_k}k}\|)&\dots &\varphi (\|\mathbf x_{{m_k}k}-\mathbf x_{{m_k}k}\|)\\
    \end{spmatrix} \\
    {\begin{spmatrix}{\mathbf w_{k}}w_{1k}\\w_{2k}\\\vdots \\w_{{m_k}k}\end{spmatrix}}= 
    {\begin{spmatrix}{\mathbf y_k}
    y_{1k}  \\
    y_{2k}  \\
    \vdots\\
    y_{{m_k}k}\end{spmatrix}},
    \label{eq:RBF}
\end{split}
$}
\end{equation}
and compute $\mathbf w_k$, by 
\[\mathbf w_k=\pmb \varphi^{-1}_k \mathbf y_k.\]

The elements $\varphi_{ij}=\varphi(\|\mathbf x_{jk}-\mathbf x_{ik}\|)$, of matrix $\pmb \varphi_k$ denotes the application of function $\varphi$ to the Euclidean Distances (or norms) of the observations in the $k^{th}$ cluster. Note that vector $\mathbf w_{k}$ has length $m_k$ for each cluster $k$, instead of $n$ for the sigmoid approach. Afterwards, similar to the sigmoid functions, we obtain the entire representation for all clusters, similar to Equations \ref{eq:out-layer},\ref{eq:calc-v}, for the weights of the output layer $\mathbf v$, by solving

\begin{equation}
    \begin{spmatrix}{\mathbf O}
        \hat{\pmb\varphi}_1 \mathbf w_1 \quad
        \hat{\pmb\varphi}_2 \mathbf w_2 \quad
        \dots \quad
        \hat{\pmb\varphi}_N \mathbf w_N \quad
        \mathbf 1
    \end{spmatrix}
    \begin{spmatrix}{\mathbf v}
        v_1 \\
        v_2 \\
        \vdots\\
        v_N \\
        b_0
    \end{spmatrix}
    =
    \begin{spmatrix}{\mathbf y}
        y_1  \\
        y_2  \\
        \vdots\\
        y_m
    \end{spmatrix},
    \label{eq:out-layer-rbf}
\end{equation}

similar to Equation \ref{eq:out-layer}, the rows of the matrices $\hat{\pmb\varphi}_1, \hat{\pmb\varphi}_2, \dots$ contain the observations of the entire sample, whereas the columns the collocation points found in each cluster. For calculation of the weights, we use $\varphi_{ij}=\varphi(\|\mathbf x_{jk}-\mathbf x_{ik}\|)$.After computating the $\mathbf w_k$ and $\mathbf v$, one may interpolate for any new $\mathbf x$ (out-of-sample), using

\[\varphi_j(\mathbf x)=\varphi(\|\mathbf x_{jk}-\mathbf x\|),\]

where $\mathbf x_{jk}$ are the RBF collocation points for the approximation, same as in Equation \ref{eq:RBF}. Hence, we may predict for out of sample observations by using

\begin{equation}
    f(\mathbf x)=\sum^N_{k=1}\left(\sum^n_{j=1} w_{jk} \varphi_j(\mathbf x)\right) v_k + b_0.
    \label{eq:rbf-total}
\end{equation}

It is important to note that kernel $\varphi$ is applied to each element of matrices $\pmb \varphi_k$ (Equation \ref{eq:RBF}), instead of the total row, as per Equation \ref{eq:system-m-cluster}. Hence we don't need the inverted $\varphi^{-1}$ (corresponding to $\sigma^{-1}$ in Equation \ref{eq:system-m-cluster}), while $\mathbf w_k$ is applied directly by multiplication. This results in convenient formulation for the approximation of the derivatives, as well as the solution of PDEs. 

Due to Mairhuber–Curtis theorem \cite{10.2307/2033359}, matrix $\mathbf \varphi$ may be singular, and one should select an appropriate kernel for the data under consideration. Some examples of radial basis kernels are the Gaussian $\varphi (r)=e^{-r^2/c^2}$, Multiquadric $\varphi (r)={\sqrt {1+(c r)^{2}}}$, etc., where $r=\|\mathbf x_{j}-\mathbf x_{i}\|$, and the shape parameter $c$ controls the width of the function. $c$ may take a specific value or be optimized, to attain higher accuracy for the particular data-set studied. Accordingly with sigmoid functions, after the computation of $\mathbf w_k$, we use Equation \ref{eq:out-layer-rbf}, to compute $\mathbf{v}$, and obtain the entire representation.

\begin{figure}[ht]
    \centering
    \includegraphics[scale=1.1]{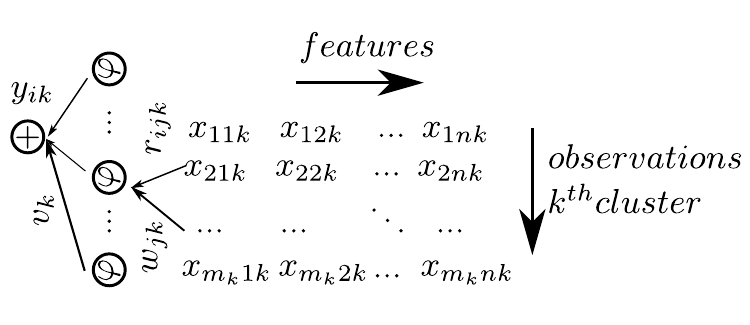}
    \caption{$k^{th}$ cluster of Radial ANNbN}
    \label{fig:radial}
\end{figure}

\subsection{ANNbN for the Approximation of Derivatives }
\label{sec:derivatives}
Equation \ref{eq:rbf-total}, offers an approximation to the sought solution, by using algebraic operations on the particular $\varphi_j (r)$, where
\begin{equation}
    r=\|\mathbf x_{j}-\mathbf x\|=\sum^{n}_{p=1}{(x_{jp}-x_p)^2},
    \label{eq:r}
\end{equation}

which is a differentiable function with respect of any out-of-sample $\mathbf x$, considering the $n$-dimensional collocation points $\mathbf x_{j}$ as constants. 

Accordingly, one may compute any higher-order derivative of the approximated function, by utilizing Equation \ref{eq:rbf-total}, and simply differentiating the kernel $\mathbf \varphi$, and multiplying by the computed weights $\mathbf w_k=w_{jk}$, for all $\mathbf x_{j}$. In particular, we may approximate the $l^{th}$ derivative with respect to the $p^{th}$ dimension, at the location of the $i^{th}$ observation  by 

\begin{equation}
\resizebox{0.55\hsize}{!}{$
    \frac{ {\partial}^l f_{i}}{\partial {x^l_{ip}}}=\begin{pmatrix}
        \frac{ {\partial}^l \varphi_{{i}{1}}}{\partial {x^l_{ip}}} \quad
        \frac{ {\partial}^l \varphi_{{i}{2}}}{\partial {x^l_{ip}}} \quad
        \dots \quad
        \frac{ {\partial}^l \varphi_{{i}{m_k}}}{\partial {x^l_{ip}}}
    \end{pmatrix}\mathbf{w_k},
    \label{eq:deriv-phi}
$}
\end{equation}
where
\[
\varphi_{{i}{j}}=\varphi_j(\mathbf x_i)=\varphi(\|\mathbf x_{jk}-\mathbf x_i\|),
\]

and
\begin{equation}
    \frac{ {\partial} \varphi_{{i}{j}}}{\partial {x_{ip}}}=
\frac{ {\partial} \varphi_{{i}{j}}}{\partial {r_{{i}{j}}}}
\frac{ {\partial} r_{{i}{j}}}{\partial {x_{ip}}},
\label{eq:partial-phi}
\end{equation}

where $\mathbf x_{jk}$ denote the collocation points of cluster $k$, and $\mathbf x_i$ the points where $f_i=f(\mathbf x_i)$ is computed. Since vector $\mathbf v$ applies by multiplication and summation to all $N$ clusters (Equation \ref{eq:rbf-total}), one may obtain the entire approximation for each partial derivative (i.e. by differentiating $\varphi$, and applying all $\mathbf w_k$ and vector $\mathbf v$, to $\frac{ {\partial}^l \varphi_{ij}}{\partial {x_{ip}}^l}$). The weights remain the same for the function and its derivatives. We should underline that the differentiation in Equation \ref{eq:partial-phi}, holds for any dimension $p\in \left\{ 1,2,\dots,n \right\}$ of $\mathbf x_i$, hence due to Equation \ref{eq:partial-phi}, with the same formulation, we derive the partial derivatives with respect to any variable, in a concise setting. 

For example, if one wants to approximate a function $f(x_1,x_2)$, and later compute its partial derivatives with respect to $x_1$, by utilizing the collocation points $\mathbf x_{j}$, we may write
\[
r=\left(\mathbf x_{j1} - x_{1}\right)^{2} + \left(\mathbf x_{j1} - x_{2}\right)^{2},
\]
and if we use as a kernel 
\[
\varphi(x_1,x_2)=-\frac{r^{4}}{4},
\]
we obtain
\[
\frac{ {\partial} \varphi (x_1,x_2)}{\partial {x_{1}}}
=-2 \left(\mathbf x_{j1} - x_{1}\right) r^{3},
\]
and hence
\[
\frac{ {\partial}^2 \varphi(x_1,x_2)}{\partial {x^2_{1}}}
=-2\left(\mathbf x_{j1} - x_{1}\right)
6 \left(\mathbf x_{j1} - x_{1}\right) 
r^{2} 
- 2 r^{3}.
\]

The variable $x_1$ may take values from the collocation points or any other intermediate point, after the weights' calculation, in order to produce predictions for out-of-sample observations. In empirical practice, we may select among the available in literature RBFs, try some new, or optimize their shape parameter $c$. In Appendix I, we also provide a simple computer code for the symbolic differentiation of any selected RBF, using SymPy \cite{meurer2017sympy} package. 

Particular interest exhibit the Integrated RBFs (IRBFs) \cite{Bakas2019,Babouskos2015,Yiotis2015}, which are formulated from the indefinite integration of the kernel, such that its derivative is the RBF $\varphi$. Accordingly, we may integrate for more than one time the kernel, to approximate the higher-order derivatives. For example, by utilizing $\text{erf}(x)=\frac{1}{\sqrt{\pi }}\int_{-x}^{x}{{{e}^{-{{t}^{2}}}}dt}$, and the two times integrated Gaussian RBF for $\varphi$ at collocation points $x_j$, 
\[
{{\varphi }_{j}}(x) =\frac{{{\text{c}}^{2}}{{\text{e}}^{\frac{-{{(x-{{x}_{j}})}^{2}}}{{{c}^{2}}}}}\text{+c}\sqrt{\text{ }\!\!\pi\!\!\text{ }}\text{ (}x-{{x}_{j}}\text{) erf}\frac{\text{(}x - {{x}_{j}}\text{)}}{c}}{2},
\]
we deduce that 
\[
 \frac{ {d} \varphi_{{j}}}{d {x}}=\frac{\text{c}\sqrt{\text{ }\!\!\pi\!\!\text{ }}\text{ erf}\frac{\text{(}x - {{x}_{j}}\text{)}}{c}}{2},
\]
and hence 
\[
\frac{ {d}^2 \varphi_{{j}}}{d {x^2}}={{e}^{-\frac{{{x-x_{j}}^{2}}}{{{c}^{2}}}}},
\]
which is the Gaussian RBF, approximating the second derivative $\ddot{f}(x)$, instead of $f(x)$.

\subsection{ANNbN for the solution of Partial Differential Equations}
\label{sec:pdes}
Similar to the numerical differentiation, we may easily apply the proposed scheme to approximate numerically the solution of Partial Differential Equations (PDEs). We consider a generic Differential operator
\[T=\sum _{l=1}^{p}g_{l}(\mathbf x)D^{l},\]
depending on the $D^{l}$ partial derivatives of the sought solution $f$, for some coefficient functions $g_{l}(\mathbf x)$, which  satisfy
\[Tf=h,\]
where $h$ may be any function in the form of $h(x_1,x_2,\dots,x_n)$. 
We may approximate $f$ by
\begin{equation}
    f=\sum\limits_{j=1}^{n}{{{w}_{j}}} \varphi_j({{x}})
    \label{eq:approx-phi}
\end{equation}

By utilizing Equation \ref{eq:deriv-phi}, we constitute a system of linear equations. Hence, the weights $w_{jk}$ may be calculated by solving the resulting system, as per Equation \ref{eq:RBF}. 

For example, consider the following generic form of the Laplace equation
\begin{equation}
    \nabla ^{2}f=h\qquad, 
    \label{eq:Lapla}
\end{equation}

\[\frac{{{\partial }^{2}}f}{\partial {{x}^{2}}}+\frac{{{\partial }^{2}}f}{\partial {{y}^{2}}}=h(x,y).\]

The weights ${w}_{j}$ in Equation \ref{eq:approx-phi} are constant, hence the differentiation regards only function $\varphi$. Thus, by writing Equation \ref{eq:Lapla} for all $h_i=h(\mathbf x_{ik})=y_{ik}$ found in cluster $k$, we obtain

\begin{equation}
\begin{split}
    \begin{spmatrix}{\mathbf{D^2}\pmb\varphi_k}
    \frac{{{\partial }^{2}}\varphi_{11}}{\partial {{x}^{2}}}+\frac{{{\partial }^{2}}\varphi_{11}}{\partial {{y}^{2}}}
    &
    \frac{{{\partial }^{2}}\varphi_{12}}{\partial {{x}^{2}}}+\frac{{{\partial }^{2}}\varphi_{12}}{\partial {{y}^{2}}}
    &\ddots \\
    \vdots & \ddots & \frac{{{\partial }^{2}}\varphi_{{m_k}{m_k}}}{\partial {{x}^{2}}}+\frac{{{\partial }^{2}}\varphi_{{m_k}{m_k}}}{\partial {{y}^{2}}}\\
    \end{spmatrix} \\
    {\begin{spmatrix}{\mathbf w_{k}}w_{1k}\\w_{2k}\\\vdots \\w_{{m_k}k}\end{spmatrix}}=
    {\begin{spmatrix}{\mathbf y_k}
    y_{1k}  \\
    y_{2k}  \\
    \vdots\\
    y_{{m_k}k}\end{spmatrix}}.
    \label{eq:RBF-PDE}
\end{split}
\end{equation}

Because the weights $w_{jk}$ are the same for the approximated function and its derivatives, we may apply some boundary conditions for the function or its derivatives $D^l$, at some boundary points $b \in \{1,2,\dots,m_b\}$, 
\[
\frac{{{\partial }^{l}}f(\mathbf x_b)}{\partial {{x}^{l}_p}}=y_b
\]

by using 
\begin{equation}
\begin{split}
    \begin{spmatrix}{\mathbf{D^l}\pmb\varphi_k}
    \frac{{{\partial }^{l}}\varphi_{11}}{\partial {{x}^{l}_p}} 
    & \frac{{{\partial }^{l}}\varphi_{12}}{\partial {{x}^{l}_p}}
    &\ddots \\
    \vdots & \ddots & \frac{{{\partial }^{l}}\varphi_{{m_b}{m_k}}}{\partial {{x}^{l}_p}}\\
    \end{spmatrix} \\
    {\begin{spmatrix}{\mathbf w_{k}}w_{1k}\\w_{2k}\\\vdots \\w_{{m_k}k}\end{spmatrix}}=
    {\begin{spmatrix}{\mathbf y_b}
    y_{1}  \\
    y_{2}  \\
    \vdots\\
    y_{{m_b}}\end{spmatrix}}.
\end{split}
\label{eq:boundary}
\end{equation}

hence, we may compute $\mathbf w_k$, by solving the resulting system of Equations
\begin{equation}
    \begin{pmatrix}
        \mathbf{D^2}\pmb\varphi_k  \\
        \mathbf{D^l}\pmb\varphi_k  \\
    \end{pmatrix}
    \mathbf w_{k}
    =
    \begin{pmatrix}
        \mathbf y_k  \\
        \mathbf y_b  \\
    \end{pmatrix},
    \label{eq:system-PDEs}
\end{equation}

similar to Equation \ref{eq:RBF} for cluster $k$. Afterwards, we may obtain the entire representation for all clusters, by using Equation \ref{eq:out-layer-rbf} for the computation of $\mathbf{v}$. Finally, we obtain the sought solution by applying the computed weights $\mathbf{w},\mathbf{v}$ in Equation \ref{eq:rbf-total}, for any new $\mathbf x$.

\subsection{Deep Networks}
\label{sec:deep}
A method for the transformation of shallow ANNbNs to Deep Networks is also presented. Although Shallow Networks exhibited vastly high accuracy even for unstructured and complex data-sets, Deep ANNbNs may be utilized for research purposes, for example in the intersection of neuroscience and artificial intelligence. After the calculation of the weights for the first layer $w_{jk}$, we use them to create a second layer (Figure \ref{fig:ANNbN}), where each node corresponds to the given $y_i$. We then use the same procedure for each neuron $k$ of layer $l\in \left\{ 2,3,\ldots ,L \right\}$, by solving:

\begin{equation}
\begin{split}
\sigma \odot \Bigg(
\begin{spmatrix}{\mathbf X}
    x_{{1}1} & x_{{1}2} & \dots & x_{{1}n} & 1\\
    x_{{2}1} & x_{{2}2} & \dots & x_{{2}n} & 1\\
    \vdots &\vdots &\ddots &\vdots & 1\\
    x_{{m}1} & x_{{m}2} & \dots & x_{{m}n} & 1\\
\end{spmatrix} \\
\begin{spmatrix}{\mathbf w_l}
    {w}_{11l} & {w}_{12l} & \dots & {w}_{1Nl} & 1\\
    {w}_{21l} & {w}_{11l} & \dots & {w}_{1Nl} & 1\\
    \vdots\\
    {w}_{n1l} & {w}_{n2l} & \dots & {w}_{nNl} & 1\\
    b_{1l} & b_{2l} & \dots & b_{Nl} & 1
\end{spmatrix} \Bigg)
\begin{spmatrix}{\hat{\mathbf v}_{lk}}
    \hat{v}_{1kl} \\
    \hat{v}_{2kl}  \\
    \vdots\\
    \hat{v}_{Nkl}\\
    b_{0kl}
\end{spmatrix}
=
\begin{spmatrix}{\mathbf {\hat{y}_k}}
    \sigma^{-1}(y_1)  \\
    \sigma^{-1}(y_2)  \\
    \vdots\\
    \sigma^{-1}(y_{m})
\end{spmatrix},
\end{split}
\label{eq:deep-full}
\end{equation}

with respect to $\hat{\mathbf v}_{lk}$. This procedure is iterated for all neurons $k$ of layer $l \in \{2,3,\dots,L\}$. Matrix $\mathbf w_l$ corresponds to the weights of layer $l-1$. Finally we calculate for the output layer, the linear weights $v_k$, as per Equation \ref{eq:out-layer}.
This procedure results in a good initialization of the weights, close to the optimal solution, and if we normalize $y_i$ in a range close to the linear part of the sigmoid function $\sigma$ (say $[0.4,0.6]$), we rapidly obtain a deep network with approximately equal errors with the shallow. Afterwards, any optimization method may supplementary applied to compute the final weights, however, the accuracy is already vastly high.

Alternatively, we may utilize the obtained layer for the shallow implementation of ANNbN, $\mathbf O$ (see Equation \ref{eq:out-layer}), as an input $x_{ij}$ for another layer, then for a third, and sequentially up to any desired number of layers.

\begin{figure*}[htb]   
\centering
\subfloat[Deep ANNbNs with $L$ layers of $N$ neurons]{\includegraphics[width=0.41\textwidth, keepaspectratio]{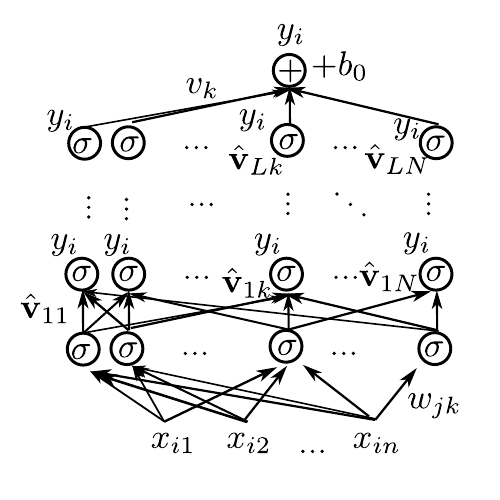}\label{fig:deep}}
\subfloat[Ensembles of $n_f$ ANNbNs]{\includegraphics[width=0.45\textwidth, keepaspectratio]{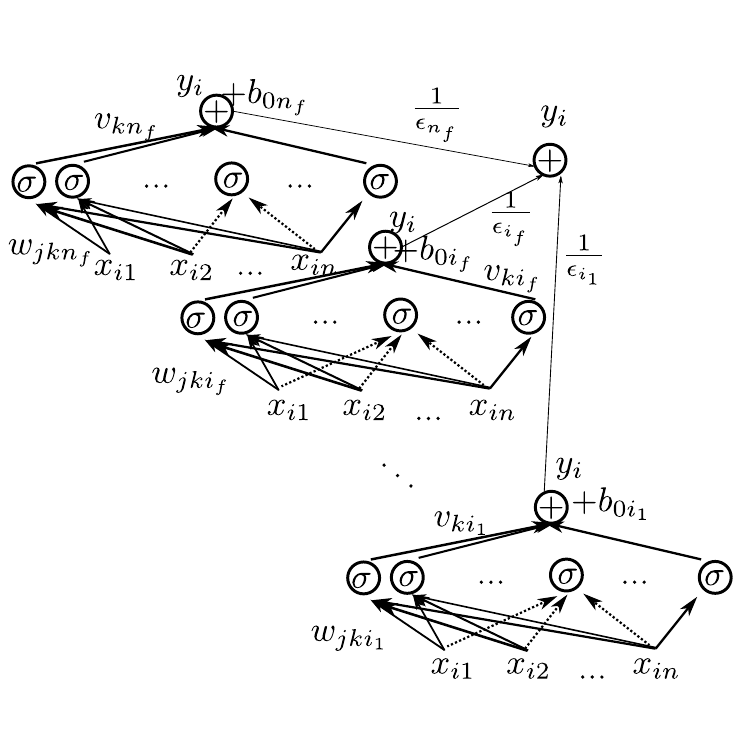}\label{fig:ensebles}}
\caption[Optional caption for list of figures 5-8]{Transformation of the basic Numerical Scheme}
\label{fig:tranform-ens-deep}
\end{figure*}

\subsection{Ensembles}
\label{sec:ensembles}
By randomly sub-sampling at a percentage of $\alpha\%$ of the observations, running the ANNbN algorithm for multiple times $i_f\in \left\{ 1,2,\ldots ,n_f \right\}$, and averaging the results with respect to the errors $\epsilon_{i_f}$ over all $n$-folds $n_f$ 

\[
y_i=\frac{\sum^{n_f}_{i_f=1} y_{i,i_f}\frac{1}{\epsilon_{i_f}}}{\sum^i_{i_f=1} \frac{1}{\epsilon_{i_f}}},
\]

we may constitute an Ensemble of ANNbN (Figure \ref{fig:ensebles}). Ensembles of ANNbNs exhibited increased accuracy and generalization properties for noisy data, as per the following Numerical Experiments. 

\subsection{Time Complexity of the ANNbN algorithm}
\label{sec:complexity}

The training of an ANN with two layers and three nodes only, is proved to be NP-Complete in \cite{blum1989training}, if the nodes compute linear threshold functions of their inputs. Even simple cases, and approximating hypothesis, results in NP-complete problems \cite{engel2001complexity}. Apart from the theoretical point of view, the slow speed of learning algorithms is a major flaw of ANNs. To the contrary, ANNbNs are fast, because the main part of the approximation regards operations with small-sized square matrices $(n+1)\times(n+1)$, with $n$ be the number of features. We provide here a theoretical investigation of ANNbNs' time complexity, which may empirically be validated by running the supplementary code. More specifically, the computational effort of ANNbNs regards the following steps.

\begin{definition}
ANNbN Training is obtained in three distinct steps: a) Clustering, b) Inversion of small-sized matrices $\mathbf X_k$ (Equation \ref{eq:system-m-cluster}) for the calculation of $w_{jk}$ weights, and c) Calculation of $v_k$ weights (Equation \ref{eq:calc-v}).
\end{definition}

\begin{definition}
Let $m$ be number of \emph{observations}, $n$ the number of \emph{features}. In the case when equally sized \emph{clusters} are used, the number of clusters $N$, which is equal to number of neurons, is
\begin{equation}
    N=\lfloor\frac{m}{n+1}\rfloor, 
    \label{eq:proof-N}
\end{equation}
\end{definition}

where the addition of $1$, corresponds to the unit column in Equation \ref{eq:system-m-cluster}. Note that the number of clusters $N$, is equal to number of neurons as well (Equation \ref{eq:system-m-cluster}, and Figure \ref{fig:ANNbN}). This is the maximum number of clusters, otherwise the matrices $\mathbf X_k$ are not invertible and Equation \ref{eq:system-m-cluster} has more than one solutions. Hence we investigate the worst case in terms of computational time, while in practice $N$ may be smaller. We assume $i$ the number of iterations needed until convergence of clustering, which in practical applications is small and the clustering fast. 

\begin{lemma}
\label{step-a}
Time complexity of step (a) is $\mathcal{O}(\log{m}-\log(n+1))$.
\end{lemma}

\begin{lemma}
\label{step-b}
Time complexity of step (b) is $\mathcal{O}(mn^2)$
\end{lemma}

\begin{proof}
Time complexity of step (b) regards the inversion of matrices with size $(n+1) \times (n+1)$ (Equation \ref{eq:system-m-cluster}). This is repeated $N$ times, hence the complexity is $\mathcal{O}(Nn^3) \leq \mathcal{O}(\frac{m}{n}n^3)=\mathcal{O}(mn^2)$
\end{proof}

\begin{lemma}
\label{step-c}
Time complexity of step (c) is $\mathcal{O}(\frac{m^3}{n^2})$
\end{lemma}

\begin{proof}
Step c regards the solution of an $m\times N$ system of Equations (Eq. \ref{eq:out-layer}). We may solve with respect to $v$, by $\mathbf v=(\mathbf O^{T}\mathbf O)^{-1} \mathbf O^{T}\mathbf y.$. Hence the complexity regards a multiplication of $\mathbf O^{T}\mathbf O$ with $\mathcal{O}(NmN)=\mathcal{O}(N^2m)$, its inversion with complexity $\mathcal{O}(N^3)$, as well multiplication of  $(\mathbf O^{T}\mathbf O)^{-1}$ with $\mathbf O^{T}$, with complexity $\mathcal{O}(NNm)$, and $(\mathbf O^{T}\mathbf O)^{-1} \mathbf O^{T}$, with $\mathbf y$, with complexity $\mathcal{O}(Nm1)$. Thus, the total complexity is $\mathcal{O}(mN^2+N^3+mN^2+mN)=\mathcal{O}(mN^2+N^3)\leq\mathcal{O}(m\frac{m^2}{n^2}+\frac{m^3}{n^3})=\mathcal{O}(\frac{m^3}{n^2})$.
\end{proof}

\begin{theorem}
\emph{(ANNbN Complexity)}
\label{Lagrange}
The running time of ANNbN algorithm is $\mathcal{O}(mn^2)+\mathcal{O}(\frac{m^3}{n^2})-\mathcal{O}(\log(n+1))$
\end{theorem}

\begin{proof}
By considering the Time Complexity of each step, 
(Lemma~\ref{step-a},\ref{step-b},\ref{step-c}), we deduce that the total coomplexity is $\mathcal{O}(\log{m}-\log(n+1))+\mathcal{O}(mn^2)+\mathcal{O}(\frac{m^3}{n^2})=\mathcal{O}(-\log(n+1))+\mathcal{O}(mn^2)+\mathcal{O}(\frac{m^3}{n^2})=\mathcal{O}(mn^2)+\mathcal{O}(\frac{m^3}{n^2})-\mathcal{O}(\log(n+1))$.
\end{proof}

\section{Validation Results}
\label{sec:Validation Results}

\subsection{1D Function approximation \& geometric point of view}

We consider a simple one dimensional function $f(x)$, with $x\in \mathbf{R}$, to present the basic functionality of ANNbNs. Because $ {\sigma^{-1}} (y)=\log \left({\frac {y}{1-y}}\right)$ is unstable for $y\to 0$, and $y \to 1$, we normalize the responses in the domain $[0.1, 0.9]$. In Figure \ref{fig:1D}, the approximation of $f(x)=0.3sin(e^{3x})+0.5$ is depicted, demonstrating the approximation by varying the number of neurons utilized in the ANNbN. We may see that by increasing the number of neurons from 2 to 4 to 8, the approximating ANNbN exhibits more curvature alterations. This complies with the Universal Approximation theorem, and offers a geometric point of view. Interestingly, the results are not affected by adding some random noise, $\epsilon \sim {\mathcal {U}}(-\frac{1}{20},\frac{1}{20})$, as the Mean Absolute Error (MAE) in this noisy data-set was $2.10E{-2}$ for the train set, and for the test set was even smaller $1.48E{-2}$, further indicating the capability of ANNbN to approximate the \textit{hidden} signal and not the noise. We should note that for noiseless data of 100 observations, and 50 neurons, the M.A.E. in the train set was $6.82E{-6}$ and in the test set $8.01E-{6}$. The approximation of the same function with Gaussian RBF, and shape parameter $c=0.01$, results in $7.52E{-8}$ M.A.E. for the train set and $1.07E{-7}$ for the test set. 

\begin{figure}[ht]   
\begin{minipage}{0.48\textwidth}
\centering
\includegraphics[width=0.95\textwidth, keepaspectratio]{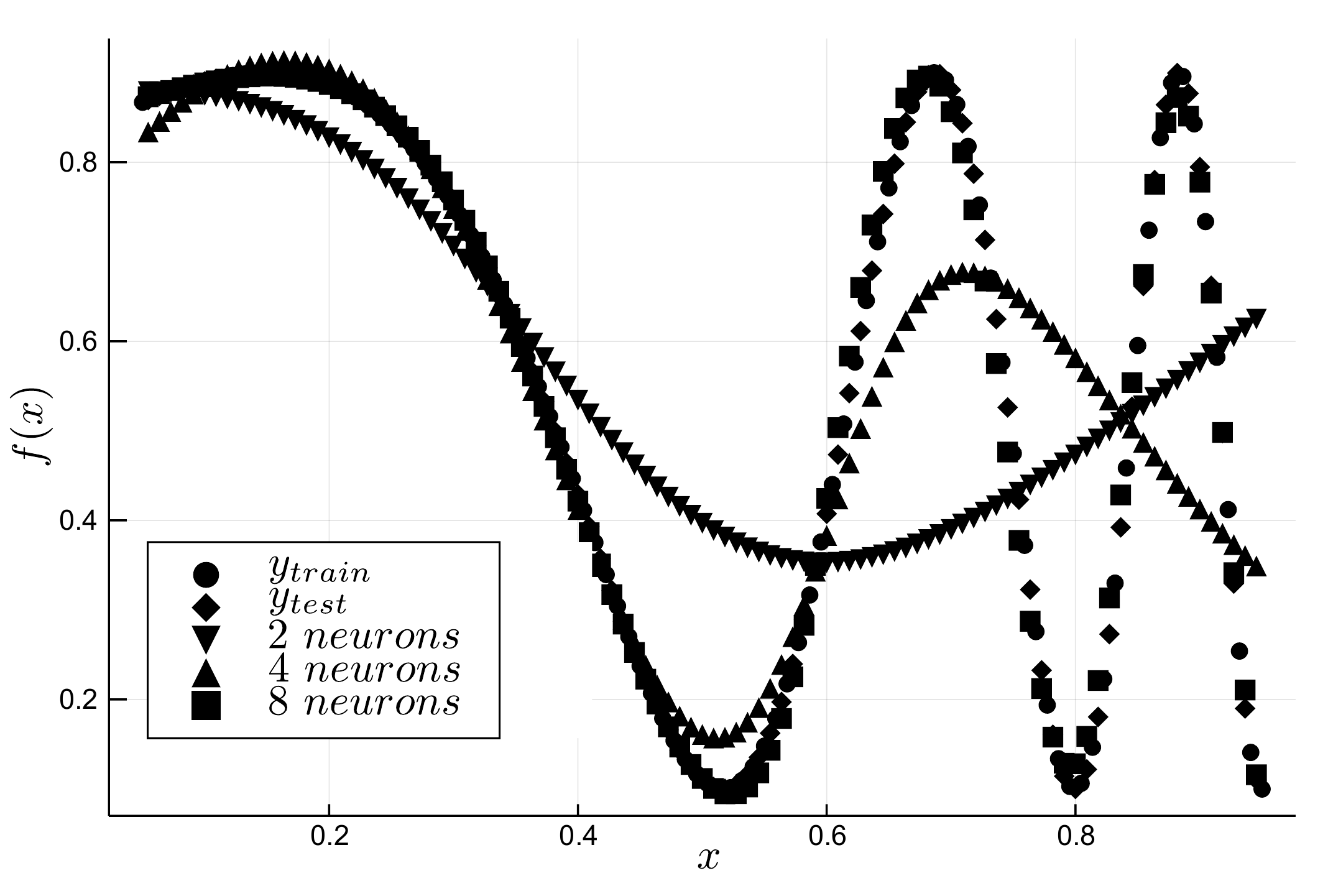}
\caption{ANNbN with $2, \ 4 \ \& \ 8$ neurons, for the approximation of $f(x)=0.3sin(e^{3x}) + 0.5$.}
\label{fig:1D}
\end{minipage}\hfill
\begin{minipage}{0.48\textwidth}
\centering
\includegraphics[width=0.95\textwidth, keepaspectratio]{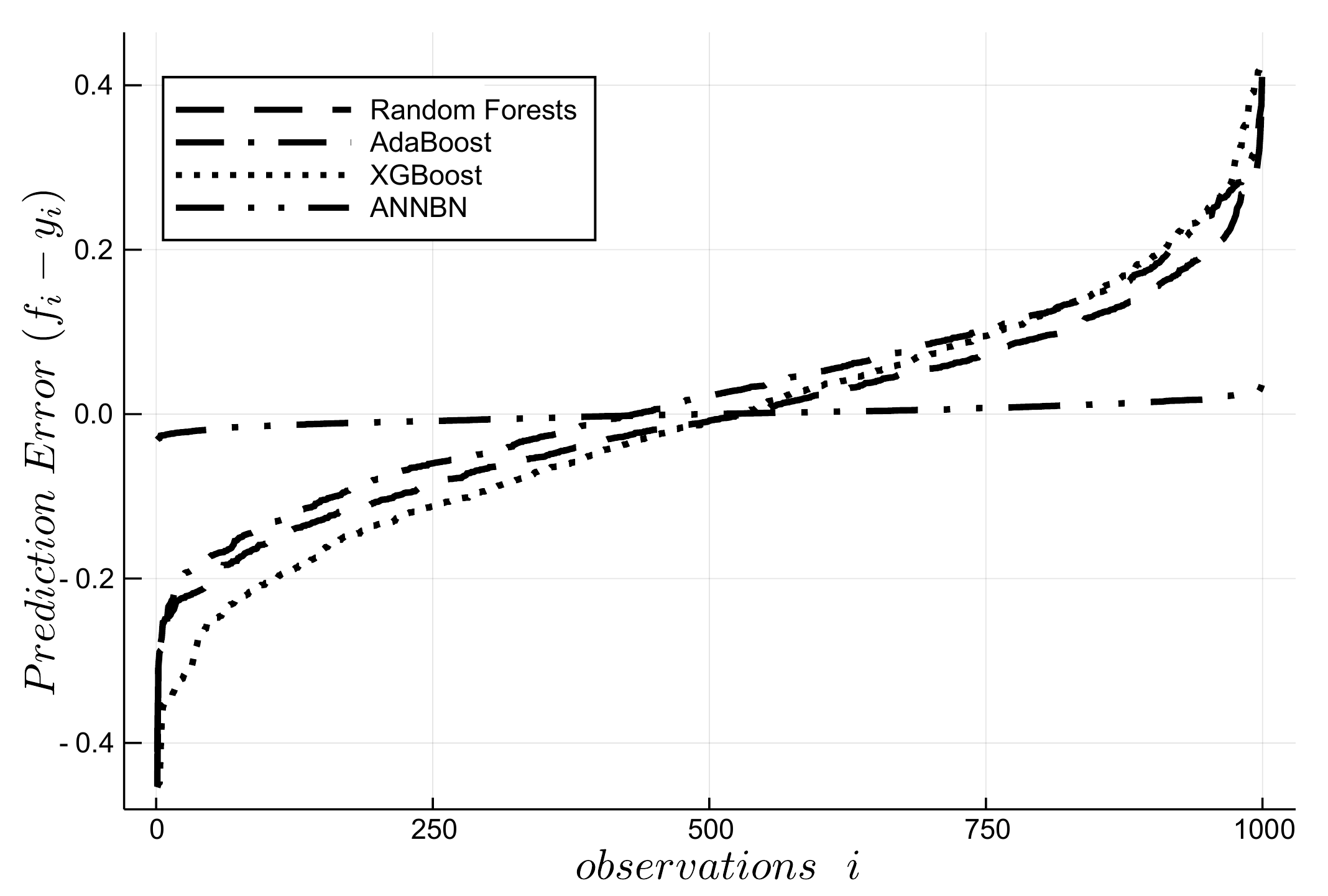}
\caption[]{Regression  Errors for the Griewank Function with input $\mathbf{x} \in \mathbb{R}^{100}$}
\label{fig:100D}
\end{minipage}
\end{figure}

\subsection{Regression in $\mathbb{R}^n$}

We consider the function of five variables,
\[
f(\mathbf x)=-x_1+\frac{x_2^2}{2}-\frac{x_3^3}{3}+\frac{x_4^4}{4}-\frac{x_5^5}{5}.
\]

We create a train set of the five variables $x_i \sim {\mathcal {U}}(\frac{1}{10},\frac{9}{10})$, compute $f_{train}$, add some random noise $\epsilon \sim {\mathcal {U}}(\frac{-1}{20},\frac{1}{20})$ and normalize $f_{train} \in [\frac{1}{10},\frac{9}{10}]$. Then we create a test set with an equal number of observations with the train set $(m=1000)$, and compute $f_{test}$, without adding random noise. Thus, we may check the capability of ANNbN to approximate the signal and not the noise. The results are presented in Table \ref{tab:Regr}, indicating great accuracy achieved with ANNbNs. The comparison with other methods regards Random Forests \cite{breiman2001random} as implemented in \cite{DecisionTree}, XGBoost \cite{XGBoost,ruder2016overview}, and AdaBoost from ScikitLearn \cite{ScikitLearn}. 

Table \ref{tab:Regr} presents the similar results -in terms of approximation errors- obtained for input $\mathbf x \in \mathbb{R}^{100}$, for $m=10000$ observations, and addition of some random noise $\epsilon \sim {\mathcal {U}}(\frac{-1}{2},\frac{1}{2})$ to the highly nonlinear Griewank function \cite{griewank1981generalized},
\[
g(\mathbf x)=1+{\frac  {1}{4000}}\sum _{{i=1}}^{n}x_{i}^{2}-\prod _{{i=1}}^{n}\cos \left({\frac  {x_{i}}{{\sqrt  {i}}}}\right) + \epsilon .
\]

With RBF ANNbNs, we may use a higher number of clusters, and hence neurons, $N > \lfloor\frac{m}{n+1}\rfloor$, as the matrices $\pmb \varphi_k$ of Equation \ref{eq:RBF}, are always square. Accordingly, we may approximate this nonlinear, noisy function with a few observations with respect to features, $(\frac{m}{n}=10)$, with vastly low errors as demonstrated in Figure \ref{fig:100D}, and Table \ref{tab:Regr}.

\begin{table}[b]   
\begin{minipage}{0.92\columnwidth}
  \centering
  \begin{threeparttable}
  \caption{Regression Results}
  \label{tab:Regr}
  \begin{tabular}{ l c c c c} \toprule
    Mean Absolute &  Random & AdaBoost &  XGBoost & ANNbN \\
    Errors & Forests & & &\\ \toprule
    $f(\mathbf x), \ \mathbf x \in \mathbb{R}^{5}$  & $2.37E{-2}$ & $3.00E{-2}$ & $3.51E{-2}$ & $4.69E{-3}$ \\  \midrule
    Griewank. $\mathbf x \in \mathbb{R}^{100}$ & $9.46E{-2}$  & $9.91E{-2}$ & $12.3E{-2}$ & $9.00E{-3}$ \\ 
 \bottomrule
  \end{tabular}
  \end{threeparttable}
\end{minipage}
\end{table}

\subsection{Classification for Computer Vision}

As highlighted in the introduction, the reproducibility of AI Research is a major issue. We utilize ANNbN for the MNIST database \cite{lecun1998gradient,mnistweb}, obtained from \cite{MLDatasets}, consisting of $6\times10^4$ handwritten integers $\in [0,9]$, for train and $10^4$ for test. The investigation regards a variety of ANNbN formulations, and the comparison with other methods. In particular, the $\text{erf}(x)=\frac{1}{\sqrt{\pi }}\int_{-x}^{x}{{{e}^{-{{t}^{2}}}}dt}$, and $\sigma=\frac{1}{1+e^{-x}}$ were utilized as activation functions, and the corresponding $erf^{-1}(x)$, and $\sigma^{-1}(x)$ for the Equation \ref{eq:system-m-cluster}.  We constructed a ANNbNs with one and multiple layers, varying the number of neurons and normalization of $y$, in the domain $\left[\epsilon, 1-\epsilon\right]$. The results regard separate training for each digit. All results in Table \ref{tab:MNIST} are obtained without any clustering. We consider as accuracy metric, the percentage of the Correct Classified (CC) digits, divided by the number of observations $m$ 
\[
    \alpha=100\frac{CC}{m}\%.
\]
This investigation aimed to compare ANNbN with standard ANN algorithms such as Flux \cite{Flux}, as well as Random Forests as implemented in \cite{DecisionTree}, and XGBoost \cite{XGBoost}. Table \ref{tab:MNIST} presents the results in terms of accuracy and computational time. The models are trained on the raw dataset, without any spatial information exploitation. The results in Table \ref{tab:MNIST} are exactly reproducible in terms of accuracy, as no clustering was utilized and the indexes are taken into account in ascending order. For example, the running time to train 5000 neurons is 29.5 seconds on average for each digit, which is fast, considering that the training regards $3925785$ weights, for $6E4$ instances and $784$ features. Also, the Deep ANNbNs with 10 layers with 1000 neurons each, are trained in the vastly short time of 91 seconds per digit on average (Table \ref{tab:MNIST}). Correspondingly, In Table \ref{tab:MNIST}, we compare the Accuracy and Running Time, with Random Forests (with $261\approx784/3$ Trees), and XGBoost (200 rounds). Future steps may include data preprocessing and augmentation, as well as exploitation of spatial information like in CNNs. Furthermore, we may achieve higher accuracy by utilizing clustering for the Neighborhoods training, Ensembles, and other combinations of ANNbNs. Also by exploiting data prepossessing and augmentation, spatial information, and further training of the initial ANNbN with an optimizer such as stochastic gradient descent. No GPU or parallel programming was utilized, which might also be a topic for future research. For example, the RBF implementation of ANNbN with clustering and $1.2\times10^4$ neurons exhibits a test set accuracy of 99.7 for digit $3$. The accuracy results regard the out of sample test set with $10^4$ digits. The running time was measured in an Intel i7-6700 CPU @3.40GHz with 32GB memory and SSD hard disk. A computer code to feed the calculated weights into Flux \cite{Flux} is provided. 

\begin{table*}[!hbt] 

\footnotesize
  \centering
  \begin{threeparttable}
  \caption{Computer Vision (MNIST)}
  \label{tab:MNIST}
  \begin{tabular}{ l c c c c c c c c c c} \toprule
  \toprule
  \multirow{2}{*}{Correct Classified (\%)} & \multicolumn{10}{c}{Digit Label} \\ \cmidrule(l){2-11} & 0 & 1 & 2 & 3 & 4 & 5 & 6 & 7 & 8 & 9
 \\ \toprule
    Random Forests \cite{DecisionTree} & $99.68$ & $99.73$ & $98.8$ & $98.59$ & $98.74$ & $98.79$ & $99.23$ & $98.91$ & $98.42$ & $98.35$ \\  \midrule
    XGBoost \cite{XGBoost} & $98.65$ & $98.81$ & $97.61$ & $97.09$ & $97.60$ & $97.98$ & $98.67$ & $97.83$ & $97.08$ & $96.99$ \\  \midrule
    Flux ANN$^1$ \cite{Flux} & $99.61$ & $99.65$ & $99.11$  & $99.12$  & $98.90$  & $98.98$  & $99.49$  & $98.94$  & $98.78$  & $98.55$ \\ \midrule
    ANNbN$^1\Diamond\blacktriangleright$ & $99.69$ & $99.74$ & $99.25$  & $99.44$  & $99.23$  & $99.27$  & $99.53$  & $99.20$  & $99.12$  & $99.01$ \\ \midrule
    ANNbN$^2\Diamond$  &$99.77$ & $99.81$ & $99.39$ & $99.36$ & $99.42$ & $99.44$ & $99.63$ & $99.3,$ & $99.19,$ & $99.05$\\ \midrule
    ANNbN$^3\Diamond$ & $99.81$ & $99.81$ & $99.42$ & $99.55$ & $99.53$ & $99.51$ & $99.66$ & $99.35$ & $99.39$ & $99.21$ \\ \midrule
    ANNbN$^4\Diamond\blacktriangle$ & $99.82$ & $99.82$ & $99.42$ & $99.54$ & $99.56$ & $99.54$ & $99.66$ & $99.35$ & $99.46$ & $99.19$ \\ \midrule
    Deep ANN$^5\Diamond$ & $99.50$ & $99.62$ & $98.81$ & $98.35$ & $98.69$ & $98.75$ & $99.29$ & $98.7$ & $98.03$ & $97.61$ \\ \toprule \toprule
  \multirow{2}{*}{Running Time (sec) } & \multicolumn{10}{c}{Digit Label}  \\ \cmidrule(l){2-11} & 0 & 1 & 2 & 3 & 4 & 5 & 6 & 7 & 8 & 9
 \\ \toprule
    Random Forests \cite{DecisionTree} & $128.5$ & $122.1$ & $178.8$ & $162.9$ & $142.4$ & $157.5$ & $159.5$ & $159.0$ & $154.6$ & $153.2$\\  \midrule
    XGBoost \cite{XGBoost} & $63.9$ & $66.0$ & $64.0$ & $64.9$ & $65.4$ & $66.4$ & $64.6$ & $64.4$ & $65.2$ & $64.8$ \\  \midrule
    Flux ANN$^1$ \cite{Flux} & $879.0$ & $882.7$ & $853.9$  & $864.0$  & $866.9$  & $856.4$  & $852.9$  & $858.9$  & $871.0$  & $862.9$ \\ \midrule
    ANNbN$^1\blacktriangleright$ & $29.8$ & $33.3$  & $29.0$  & $30.6$  & $28.3$  & $27.1$  & $29.4$ & $29.9$  & $28.4$  & $28.9$ \\ \midrule
    ANNbN$^2$ & $51.6$ & $51.0$ & $51.2$ & $50.8$ & $51.8$ & $51.5$ & $52.4$ & $51.7$ & $52.0$ & $53.3$ \\ \midrule
    ANNbN$^3$ & $92.5$ & $91.5$ & $92.3$ & $92.5$ & $90.6$ & $92.6$ & $93.1$ & $92.8$ & $93.5$ & $93.1$\\ \midrule
    ANNbN$^4\blacktriangle$ & $97.3$ & $94.3$ & $94.2$ & $94.9$ & $94.7$ & $94.8$ & $94.7$ & $94.9$ & $95.0$ & $94.9$\\ \midrule
    Deep ANN$^5$ & $70.5$ & $81.3$ & $131.9$ & $66.6$ & $126.9$ & $182.9$ & $49.7$ & $90.0$ & $44.8$ & $64.0$ \\ \bottomrule
    \bottomrule
  \end{tabular}
  \begin{tablenotes}
      \small
      \item $^1$ 1 hidden layer with $5000$ Neurons, Activation Function (AF) $\text{erf}(x)=\frac{1}{\sqrt{\pi }}\int_{-x}^{x}{{{e}^{-{{t}^{2}}}}dt}$, and $\epsilon=0.00$. 
      \item $^2$ 1 hidden layer with $5000$ Neurons, AF $\sigma=\frac{1}{1+e^{-x}}$, and $\epsilon=0.01$.
      \item $^3$ 1 hidden layer with $7000$ Neurons, AF $\sigma$, and $\epsilon=0.01$. 
      \item $^4$ 1 hidden layer with $7000$ Neurons, AF $\sigma$, and $\epsilon=0.02$.
      \item $^5$ 10 hidden layes with $1000$ Neurons each, AF $\sigma$, and $\epsilon=0.02$.
      \item $\blacktriangleright$ fastest design; $\blacktriangle$ highest accuracy.
      \item $\Diamond$ The ANNbN accuracy results, are exactly reproducible with the supplementary Computer Code. All training examples utilize the raw MNIST database, without any preprocessing or data augmentation. The accuracy (\%) regards the out-of-sample test set of MNIST with $10^4$ handwritten digits.
    \end{tablenotes}
  \end{threeparttable}

\end{table*}

\subsection{Solution of Partial Differential Equations}
\label{subsec:Solution of PDEs}
We consider the Laplace's Equation \cite{Brady2005}

\[
    \frac{{{\partial }^{2}}f}{\partial {{x}^{2}}}+\frac{{{\partial }^{2}}f}{\partial {{y}^{2}}}=0,
\]

in a rectangle with dimensions (a,b), and boundary conditions $f(0,y)=0$ for $y \in [0,b]$, $f(x,0)=0$, for $x \in [0,a]$, $f(a,y)=0$  , for $y \in [0,b]$, and $f(x,b)=f_0sin(\frac{\pi}{a}x$), for $x \in [0,a]$. In Figure \ref{fig:PDE-Laplace}, the numerical solution as well as the exact solution
\[f(x)=\frac{f_0}{sinh(\frac{\pi}{a}b)} sin(\frac{\pi}{a}x)sinh(\frac{\pi}{a}y),\]

are presented. The MAE among the closed-form solution and the numerical with ANNbN, was found $3.97E{-4}$. Interestingly, if we add some random noise in the zero source; i.e. 
\begin{equation}
    \frac{{{\partial }^{2}}f}{\partial {{x}^{2}}}+\frac{{{\partial }^{2}}f}{\partial {{y}^{2}}}=\epsilon \sim {\mathcal {U}}(0,\frac{1}{10}),
    \label{eq:laplace}
\end{equation}

the MAE remains small, and in particular $2.503E-{3}$, for a=b=1, over a rectangle grid of points with $dx=dy=0.02$. It is important to underline that numerical methods for the solution of partial differential equations are highly sensitive to noise \cite{mai2003approximation,Bakas2019}, as it vanishes the derivatives. However, by utilizing the ANNbN solution the results change slightly, as described in the above errors. This is further highlighted if we utilize the calculated weights of the ANNbN approximation and compute the partial derivatives of the solution $f$ of Equation \ref{eq:laplace}, $\frac{{{\partial }^{2}}f}{\partial {{x}^{2}}}$, and $\frac{{{\partial }^{2}}f}{\partial {{y}^{2}}}$, the corresponding MAE for the second order partial derivatives is $6.72E{-4}$ (Figure \ref{fig:PDE-Laplace-der}), which is about two orders less than the added noise $E({\mathcal {U}}(0,\frac{1}{10}))=0.05$, implying that ANNbN approximates the signal and not the noise even in PDEs, and even with a stochastic source.

\begin{figure*}[!ht]   
\centering
\subfloat[Solution]{\includegraphics[width=0.45\textwidth, keepaspectratio]{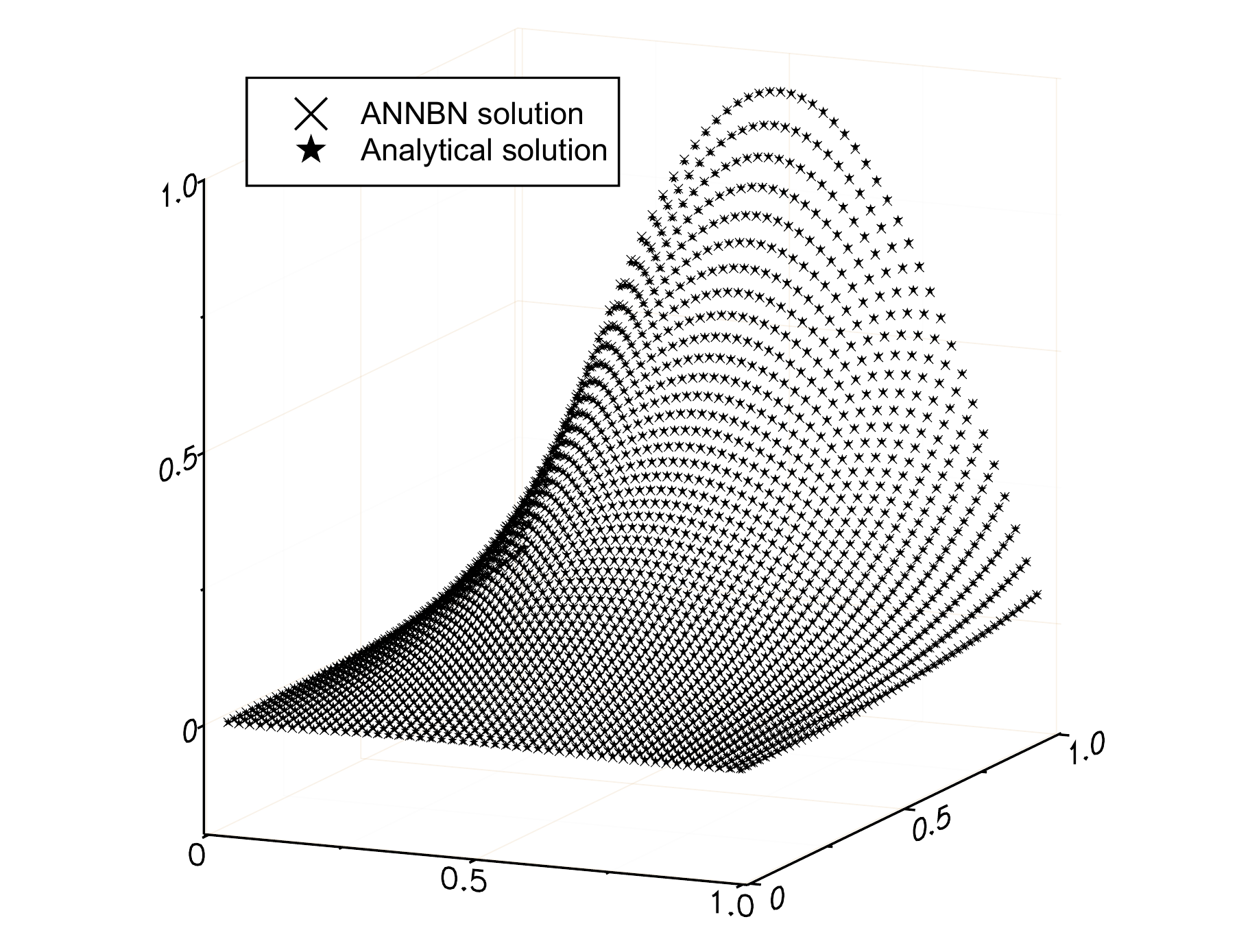}\label{fig:PDE-Laplace}}
\subfloat[Partial Derivatives]{\includegraphics[width=0.45\textwidth, keepaspectratio]{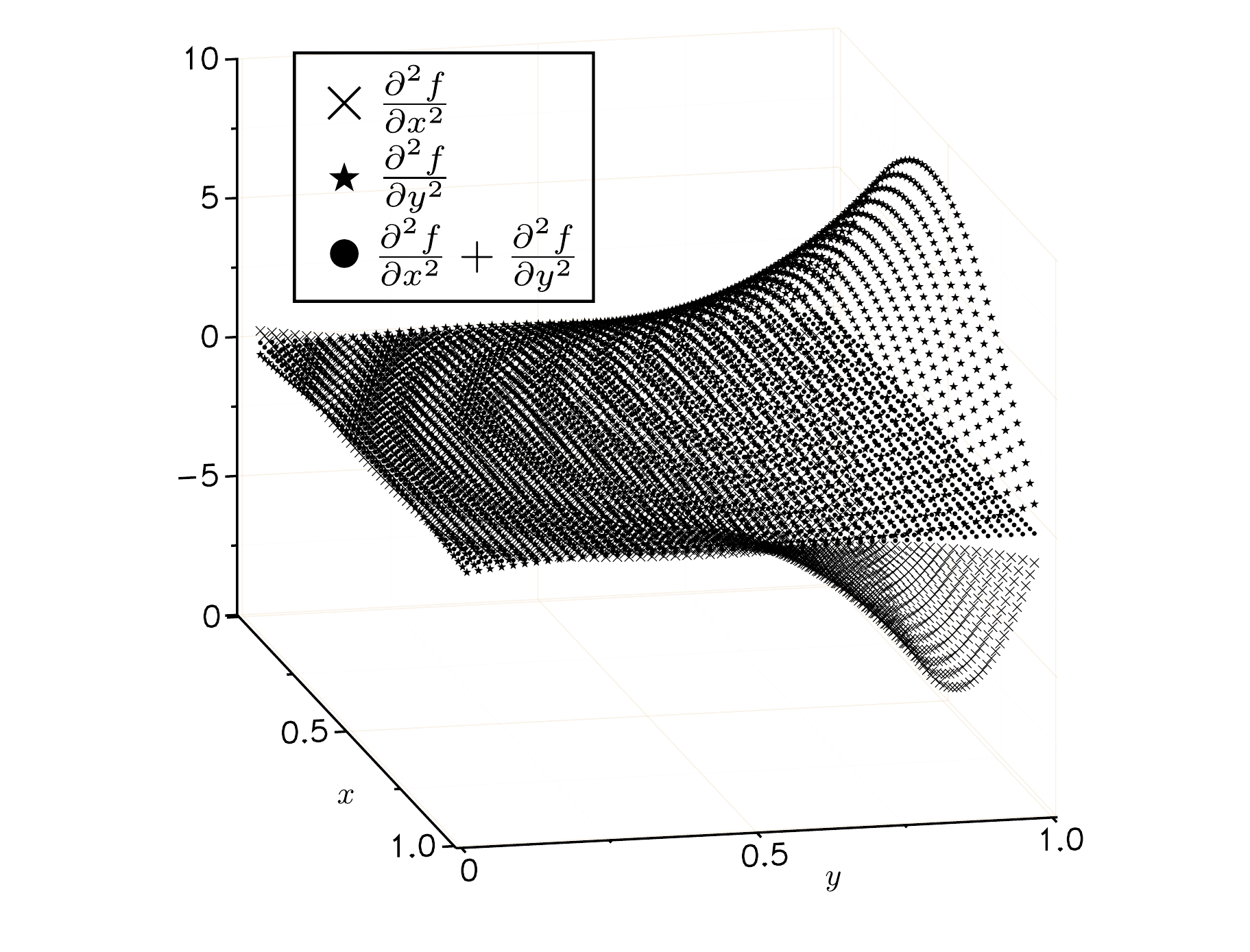}\label{fig:PDE-Laplace-der}}
\caption[Optional caption for list of figures 5-8]{ANNbN solution of Laplace's Equation with stochastic source.}
\label{fig:PDE}
\end{figure*}

\section{Discussion and Conclusions}
\label{sec:Discussion}
As described in the formulation of the proposed method, we may use a variety of ANNbNs, such as Sigmoid or Radial Basis Functions scheme, Ensembles of ANNbNs, Deep ANNbNs, etc. The method adheres to the theory of function approximation with ANNs, as per Visual representations of ANNs' capability to approximate continuous functions \cite{Nielsen2015,rojas2013neural}. We explained the implementation of the method in the presented illustrative examples, which may be reproduced with the provided computer code. In general, Sigmoid functions are faster, RBFs more accurate and Ensembles of either sigmoid of RBFs handle better the noisy datasets. RBFs, may use smaller than $N=\lfloor\frac{m}{n+1}\rfloor$ sized matrices, and hence approximate datasets with limited observations and a lot of features. The overall results are stimulating in terms of speed and accuracy, compared with state-of-the-art methods in the literature. 

The approximation of the partial derivatives and solution of PDEs, with or without noisy source, in a fast and accurate setting, offers a solid step towards the unification of Artificial Intelligence Algorithms with Numerical Methods and Scientific Computing. Future research may consider the implementation of ANNNs to specific AI applications such as Face Recognition, Reinforcement Learning, Text Mining, etc., as well as Regression Analyses, Predictions, and solutions of other types of PDEs. Furthermore, the investigation of other sigmoid functions than the logistic, such as $\tanh, \arctan, \text{erf}, \text{softmax}$, etc., as well as other RBFs, such as multiquadrics, integrated, etc., and the selection of an optimal shape parameter for even higher accuracy, are also of interest. Finally, although the weights' computation is whopping fast, the algorithm may easily be converted to parallel, as the weights' computation for each neuron regards the $N$ times inversion of matrices $\mathbf X_k$.

Interpretable AI is a modern demand in Science, and ANNbNs are inherently suitable for this purpose, as by checking the approximation errors of the neurons in each cluster, one may retrieve information for the local accuracy, as well as local and global non-linearities in the data properties. Furthermore, as demonstrated in the examples, the method is proficient for small datasets, without over-fitting, by approximating the signal and not the noise, which is a common problem of ANNs.


\section{Appendix I: Computer Code}
The presented method is implemented in Julia \cite{bezanson2017julia} Language. The corresponding computer code, is available on 
\url{https://github.com/nbakas/ANNbN.jl}

\nomenclature{${{w}_{jk}}$}{weights of first layer for $j^{th}$ neuron, in $k^{th}$ cluster}
\nomenclature{$m$}{number of observations}
\nomenclature{$n$}{number of variables (features)}
\nomenclature{${{x}_{ij}}$}{given data $(m\times n)$} 
\nomenclature{$N$}{number of clusters, equal to number of neurons}
\nomenclature{${{v}_{k}},{{b}_{0}}$}{approximation weights and bias for the external layer}
\nomenclature{$y_{ik}$}{local observations in $k^{th}$ cluster}
\nomenclature{$m_k$}{number of observations found in $k^{th}$ cluster}
\nomenclature{$\mathbf X_k$}{local matrix of observations (in $k^{th}$ cluster)}
\nomenclature{$x_{ijk}$}{observations in the $k^{th}$ cluster}
\nomenclature{$i,j,k$}{iterators for observations, features, and clusters}
\nomenclature{$\mathbf x$}{out-of-sample point in $\mathbb{R}^n$}
\nomenclature{$CC$}{Correct Classified Observations}
\nomenclature{$n_f$}{number of folds for Ensembles}
\nomenclature{$b_k$}{bias for $k^{th}$ neuron}
\nomenclature{$L$}{number of Layers}

\printnomenclature[1cm]

\bibliographystyle{IEEEtran}
\bibliography{refs}

\end{document}